\documentclass[letterpaper, 10 pt, journal, twoside]{IEEEtran}
\ifCLASSINFOpdf
  \usepackage[pdftex]{graphicx}
\else
  \usepackage[dvips]{graphicx}
\fi
\usepackage{amsmath,amsfonts,amssymb}
\usepackage{amsthm}  {
      \theoremstyle{plain}
      \newtheorem{definition}{Definition}
      \newtheorem{lemma}{Lemma}
      \newtheorem{theorem}{Theorem}

      \newtheorem{proposition}{Proposition}
      
}
\usepackage{algpseudocode}
\usepackage{algorithm}

\ifCLASSOPTIONcompsoc
 \usepackage[caption=false,font=normalsize,labelfont=sf,textfont=sf]{subfig}
\else
 \usepackage[caption=false,font=footnotesize]{subfig}
\fi
\usepackage{nicefrac}

\usepackage{booktabs}
\usepackage{multirow}

\usepackage[version-1-compatibility]{siunitx}
\DeclareMathOperator*{\argmax}{argmax}

\begin{document}
%
\title{Multi-Sensor Next-Best-View Planning as Matroid-Constrained Submodular Maximization}
\author{Mikko Lauri$^{1}$, Joni Pajarinen$^{2}$, Jan Peters$^{3}$, and Simone Frintrop$^{1}$
\thanks{Manuscript received: February 24, 2020; Revised May 22, 2020; Accepted June 19, 2020.}
\thanks{This paper was recommended for publication by Editor Tamim Asfour upon evaluation of the Associate Editor and Reviewers' comments.
This work was supported by European Research Council Grant No. 640554 (SKILLS4ROBOTS) and German Research Foundation project PA 3179/1-1 (ROBOLEAP).} 
\thanks{$^{1}$Mikko Lauri and Simone Frintrop are with Department of Informatics, University of Hamburg, Germany
        {\tt\footnotesize \{lauri,frintrop\}@informatik.uni-hamburg.de}}%
\thanks{$^{2}$Joni Pajarinen is with Intelligent Autonomous Systems lab, TU Darmstadt, Germany and with Tampere University, Finland
        {\tt\footnotesize pajarinen@ias.tu-darmstadt.de}}%
\thanks{$^{3}$Jan Peters is with Intelligent Autonomous Systems lab, TU Darmstadt, Germany and with Max Planck Institute for Intelligent Systems, Germany
  {\tt\footnotesize peters@ias.tu-darmstadt.de}}%
}
%
%

\markboth{IEEE Robotics and Automation Letters. Preprint Version. Accepted June, 2020}
{
Multi-Sensor Next-Best-View Planning as Matroid-Constrained Submodular Maximization} 

%



\maketitle

\begin{abstract}
3D scene models are useful in robotics for tasks such as path planning, object manipulation, and structural inspection.
We consider the problem of creating a 3D model using depth images captured by a team of multiple robots.
Each robot selects a viewpoint and captures a depth image from it, and the images are fused to update the scene model.
The process is repeated until a scene model of desired quality is obtained.
Next-best-view planning uses the current scene model to select the next viewpoints.
The objective is to select viewpoints so that the images captured using them improve the quality of the scene model the most.
In this paper, we address next-best-view planning for multiple depth cameras.
We propose a utility function that scores sets of viewpoints and avoids overlap between multiple sensors.
We show that multi-sensor next-best-view planning with this utility function is an instance of submodular maximization under a matroid constraint.
This allows the planning problem to be solved by a polynomial-time greedy algorithm that yields a solution within a constant factor from the optimal.
We evaluate the performance of our planning algorithm in simulated experiments with up to 8 sensors, and in real-world experiments using two robot arms equipped with depth cameras.
\end{abstract}


%
\IEEEpeerreviewmaketitle

\section{Introduction}
\label{sec:intro}
\IEEEPARstart{S}{cene} reconstruction is the process of creating a digital model of a real-world scene from a set of images or other measurements of the scene.
Models obtained via scene reconstruction are useful for robotic applications such as object manipulation, structural inspection~\cite{Quenzel2019}, and waste sorting (Fig.~\ref{fig:online_scene}).
In an online setting, the model reconstructed from the images captured so far is used to plan from which viewpoint the next image is captured.
Next-best-view (NBV) planning~\cite{Connolly1985} determines the next viewpoint that provides the greatest improvement to the quality of the current model, reducing the amount of time and number of images required to reconstruct a model of desired quality.

Many tasks may benefit from deployment of a multi-robot team~\cite{Best2018}, however most approaches to NBV planning focus on the single-robot setting~\cite{Connolly1985,Banta2000,Kriegel2015,Vasquez-Gomez2017,delmerico2018comparison,Border2018}.
In a multi-robot setting these approaches would plan the viewpoint for each robot individually and ignore coordination between team members.
Time and resources are wasted if the same part of the scene is observed by multiple robots.

\begin{figure}[t]
  \centering
  \includegraphics[width=\columnwidth]{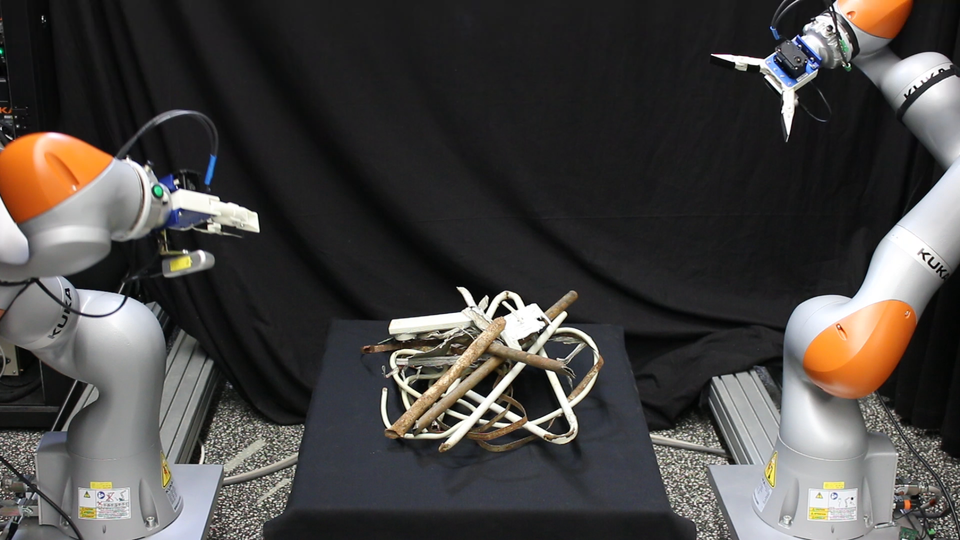}
  \includegraphics[width=0.65\columnwidth]{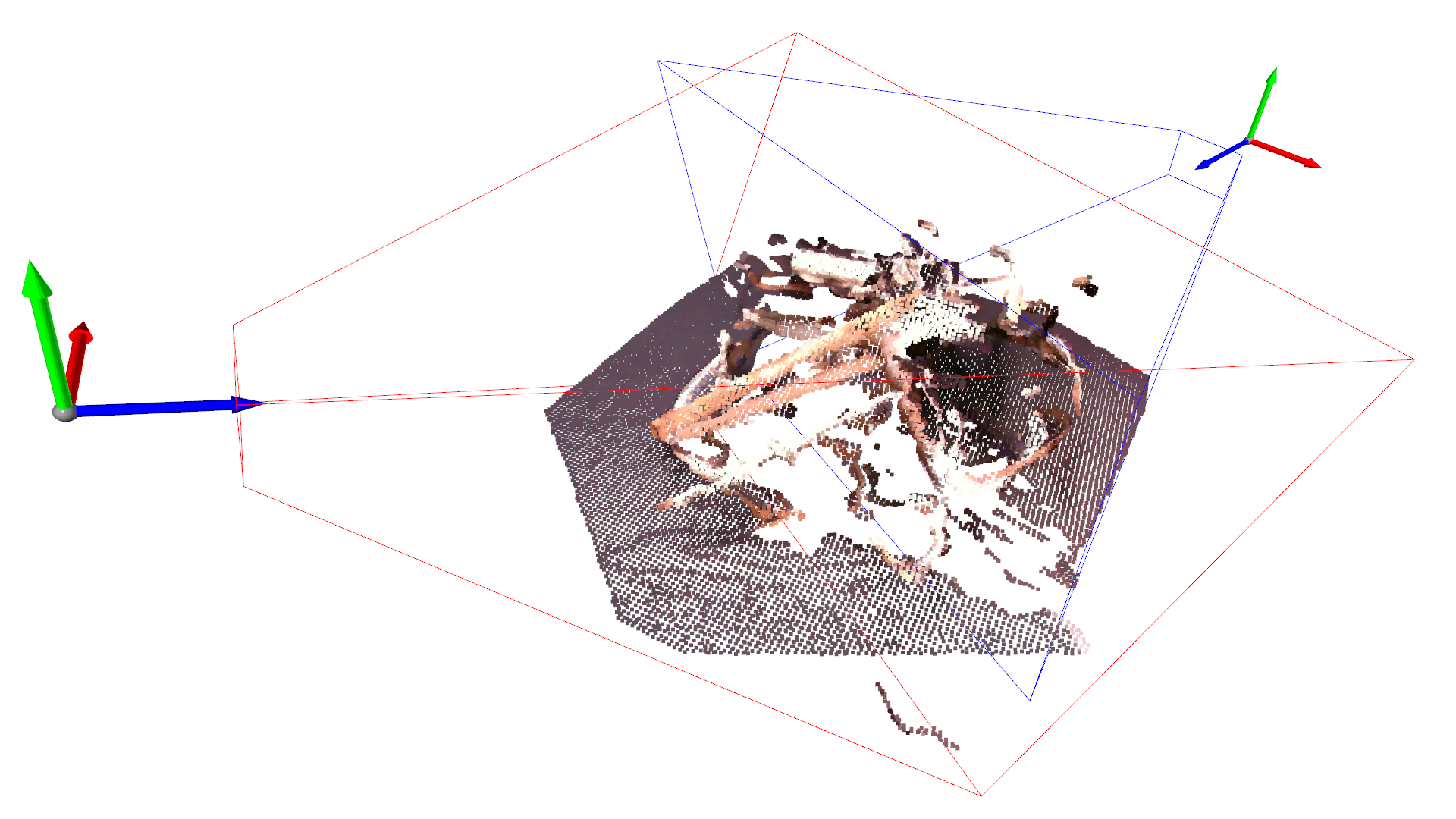}
  \caption{Above: two robots sorting waste. Below: The robots create a 3D model by recording images from views around the scene (red and blue frustums). We propose an efficient method to find the next-best-views for multiple cameras based on greedy maximization of a submodular utility function. The views found by our algorithm avoid overlap between the sensors' fields of view.}
  \label{fig:online_scene}
\end{figure}

In this paper, we propose an efficient method for multi-sensor NBV planning that coordinates view selection and avoids overlapping views (Fig.~\ref{fig:online_scene}).
The constraint that each robot can choose at most one view is formalized as a \emph{matroid}, a mathematical object that generalizes the concept of independence from linear algebra to sets.
View selection and other related sensor selection problems often satisfy a diminishing returns property known as \emph{submodularity}~\cite{Krause2008}: the benefit of adding a new view decreases with an increasing number of existing views.
An advantage of submodular maximization under a matroid constraint is that a polynomial-time greedy algorithm provides a solution within a constant factor from the optimum~\cite{fisher1978analysis}.

Unlike~\cite{Sukkar2019} who reconstruct regions of interest when communication is not constant, we aim to produce a complete scene reconstruction when robots communicate to coordinate their actions.
Our approach is best suited for robots deployed near each other (Fig.~\ref{fig:online_scene}) with reliable communication.
Cui et~~al.~\cite{Cui2019} consider a single robot equipped with many sensors, one of which is chosen to be active at any time.
In contrast, we consider joint planning where all sensors are operated simultaneously and the combination of views that is most useful is selected.
We apply matroid-constrained submodular maximization similar to the multi-robot planning methods in~\cite{Singh2009,Corah2019,Liu2019}.
We are the first to address multi-sensor NBV planning using submodularity.

Our key technical contribution is a utility function that avoids overlap between the views of multiple sensors.
We prove that the utility function is monotonically increasing and submodular, and formulate multi-sensor NBV planning as submodular maximization under a matroid constraint.
We describe an efficient greedy algorithm for the planning problem with an approximation guarantee.
When the sensors' potential views are disjoint, our formulation reduces to solving independent single-sensor NBV planning problems.
We experimentally verify the effectiveness of our proposed approach in a set of simulated experiments with up to 8 sensors, and in a real-world experiment with two robot arms.

The rest of the paper is organized as follows.
We review related work in Section~\ref{sec:related}, and define some useful mathematical concepts in Section~\ref{sec:preliminaries}.
We formulate the multi-sensor NBV problem in Section~\ref{sec:multi_sensor_next_best_view_planning}, and contrast it to single-sensor NBV planning.
In Section~\ref{sec:overlap_aware_utility_function_for_sets_emitted_rays} we introduce our proposed utility function, and prove that it is submodular.
Section~\ref{sec:greedy_maximization_for_multi_sensor_nbv_planning} proposes a greedy algorithm for multi-sensor NBV planning.
In Sections~\ref{sec:simulation_experiments} and~\ref{sec:real_world_experiments} we report results from simulated and real-world experiments.
Section~\ref{sec:conclusion} concludes the paper.

\section{Related work}
\label{sec:related}
We review related work in two areas: NBV planning in single and multi-sensor settings, and multi-robot planning using matroid constraints and submodularity.

\subsection{Next-best-view planning}
We focus on NBV planning approaches with a volumetric scene representation.
A volumetric scene representation consists of a finite set of grid cells, or voxels.
For each voxel, an occupancy probability that gives the probability of the voxel containing an obstacle is maintained.
In NBV planning, the set of visible voxels from each candidate view is first estimated by applying raytracing.
A score for the candidate view is calculated as the sum of per-voxel scores for each visible voxel.
The NBV with the greatest score is selected.

NBV planning methods differ in how the score of a view is calculated.
Counting measures~\cite{Connolly1985,Banta2000} count the total number of unknown visible voxels.
Probabilistic measures employ quantities such as entropy of voxel occupancy, or the visibility probability of voxels.
For example, \cite{Kriegel2015} selects the view that has the greatest average occupancy entropy in the visible voxels.
In~\cite{delmerico2018comparison}, scores that weight per-voxel scores by the probability of the voxel being visible are proposed.
Views that observe the most boundary voxels between known and unknown space are preferred in~\cite{Vasquez2014}.
In~\cite{Vasquez-Gomez2017}, the method is extended to consider the uncertainty in sensor motion.
Recently,~\cite{Hepp2018} proposes to learn scoring of candidate views by using a 3D convolutional neural network with a multi-scale volumetric map representation.
The assumption that voxel occupancies are independent is relaxed in~\cite{Hou2019} by applying Markov chain Monte Carlo.

Instead of a voxel grid, \cite{Border2018} proposes an implicit surface density representation.
The representation consists of points observed on a surface.
Regions with a high density of points are classified as core, and other regions as outliers.
The expected observed volume between core and outlier regions is maximized.
A Gaussian process (GP) implicit surface representation is used in~\cite{Hollinger2013}.
The variance of the GP quantifies the uncertainty of the surface reconstruction.
A planning algorithm uses the variance to find trajectories that reduce uncertainty the~most.

The related problem of planning how a robot should manipulate an object within the view of a stationary camera is investigated in~\cite{Krainin2011}.
The object model is a signed distance function on a voxel grid.
In~\cite{Bircher2016}, trajectories for environment exploration are scored by estimating the unknown volume visible along the trajectory.
A robotic system for structural inspection is demonstrated in~\cite{Quenzel2019}.
As a known environment is considered, an inspection path is planned offline to guarantee a desired amount of overlap between captured images. 

Some recent works consider NBV planning in the multi-robot setting.
Sukkar et~al.~\cite{Sukkar2019} reconstruct regions of interest (ROIs) by controlling viewpoints of multiple robots equipped with cameras.
ROIs correspond to fruit in an agricultural application and are detected by color thresholding.
Candidate viewpoints are scored based on the expected information gain on the detected ROIs.
The algorithm proposed in~\cite{Best2018} is applied to decentralize the planning task so that each robot can plan its views without needing to constantly communicate with the other robots.
Cui et~al.~\cite{Cui2019} select if a robot should apply a laser range finder or depth camera next.
Candidate sensing actions are scored by a weighted sum of the number of unknown voxels, occupied voxels, and voxels neighbouring a free voxel.
Unlike~\cite{Sukkar2019}, we target tasks where the robots communicate constantly to coordinate views.
We do not focus on ROIs, but strive for a complete reconstruction.
Different to~\cite{Cui2019}, we choose the views of multiple sensors simultaneously.
By applying submodular maximization, we provide a performance guarantee for our planning algorithm.

\subsection{Matroid constraints and submodularity in planning}
Multi-robot coordination may be viewed as an item selection task.
Each robot selects an item (e.g., a trajectory), with the objective of maximizing a performance measure that is a function of the selected set of items.
A matroid defines a system of independent sets, which models constraints such that one robot may select at most one trajectory.
A submodular function defined on independent sets has a diminishing returns property that states that adding a new item to an existing set of items is less useful the larger the existing set is.
Maximizing a submodular function under a matroid constraint by a polynomial-time greedy algorithm provides a constant-factor approximation~\cite{fisher1978analysis}.
Furthermore, many information-theoretic functions are submodular, making greedy submodular maximization a popular approach for tasks such as planning sensor placements~\cite{Krause2008}.

Multi-robot information gathering tasks such as exploration~\cite{Corah2019} and informative path planning~\cite{Singh2009} have been addressed as submodular maximization under a matroid constraint.
A sequential greedy allocation (SGA) algorithm is proposed in~\cite{Singh2009}, allowing extension of single-robot planning to any number of robots while maintaining performance guarantees.
A distributed variant of SGA proposed in~\cite{Corah2019} scales up to larger problems due to distribution of the computation, while maintaining approximation guarantees.
In~\cite{Liu2019}, a greedy algorithm is proposed for coupled problems, such as selecting the composition of a multi-robot team and subsequently planning how the robots should act.
The work most similar to ours is~\cite{Corah2019}, where the distributed approach incurs an additional suboptimality penalty.
In our paper we target the centralized setting where greedy submodular maximization enjoys a tighter suboptimality bound.

\section{Preliminaries} 
\label{sec:preliminaries}
We define now mathematical concepts that are later used.

\subsection{Partition matroids}
Matroids generalize the concept of independence from linear algebra to sets.
For an introduction to matroids, we refer the reader to~\cite{Oxley2003}.
For the purposes of this paper, we only require the concept of a partition matroid.

Let $A_i$ be $n$ pairwise disjoint sets, and let $a_i$ be integers s.t.~$0 \leq a_i \leq |A_i|$.
Denote by $\Omega$ the union of all $A_i$.
Then $(\Omega, \mathcal{I})$ is a partition matroid if $\mathcal{I} = \left\lbrace I \subseteq \Omega \mid \forall i: |I\cap A_i| \leq a_i \right\rbrace$.
The elements of $\mathcal{I}$ are \emph{independent sets}.
The sets $A_i$ are \emph{blocks} of the partition matroid.
We assume $a_i = 1$ for all partition matroids, such that there is at most one element per block in an independent set.

\subsection{Submodularity}
Submodularity formalizes the notion of diminishing returns: the marginal utility of adding a new item to an existing set of items is smaller the larger the existing set is.

A set function $f:2^\Omega \to \mathbb{R}$ is submodular if for any $A \subseteq B\subseteq \Omega$, and any $x \in \Omega\setminus B$, $f(A\cup \{x\})-f(A) \geq f(B\cup \{x\}) - f(B)$.
Additionally, $f$ is monotonically increasing, if for any $A\subseteq \Omega$ and $x\in\Omega\setminus A$, $f(A \cup \{x\}) \geq f(A)$.

\section{Multi-sensor next-best-view planning} 
\label{sec:multi_sensor_next_best_view_planning}
Consider a team of $n$ sensors, and let $X_1, \ldots, X_n$ be $n$ pairwise disjoint sets with $X_i$ representing the possible views for sensor $i$. 
Denote by $\mathcal{I}$ the collection of independent sets in a partition matroid with blocks $X_i$.
The problem of multi-sensor NBV planning is to select an independent set of views -- one for each sensor -- that maximizes the utility of the views.
Utility is measured by a function $f$ mapping independent sets to a real number.
Formally, the problem is
\begin{equation}
\label{eq:problem}
  \max\limits_{I \in \mathcal{I}} f(I).
\end{equation}
Next, we introduce our environment and sensor models.
We then discuss how utility is measured in single-sensor NBV planning, and motivate the need for a utility function $f$ specifically designed for the multi-sensor setting.

\subsection{Environment and sensor models} 
\label{sub:environment_and_sensor_models}
The environment is represented as a volumetric grid $V$ with a finite number of grid cells.
For each grid cell, $P(v)$ gives the probability that the grid cell is occupied by an obstacle.
The cell occupancies are independent of each other.

\begin{figure}
  \centering
  \includegraphics[width=\columnwidth]{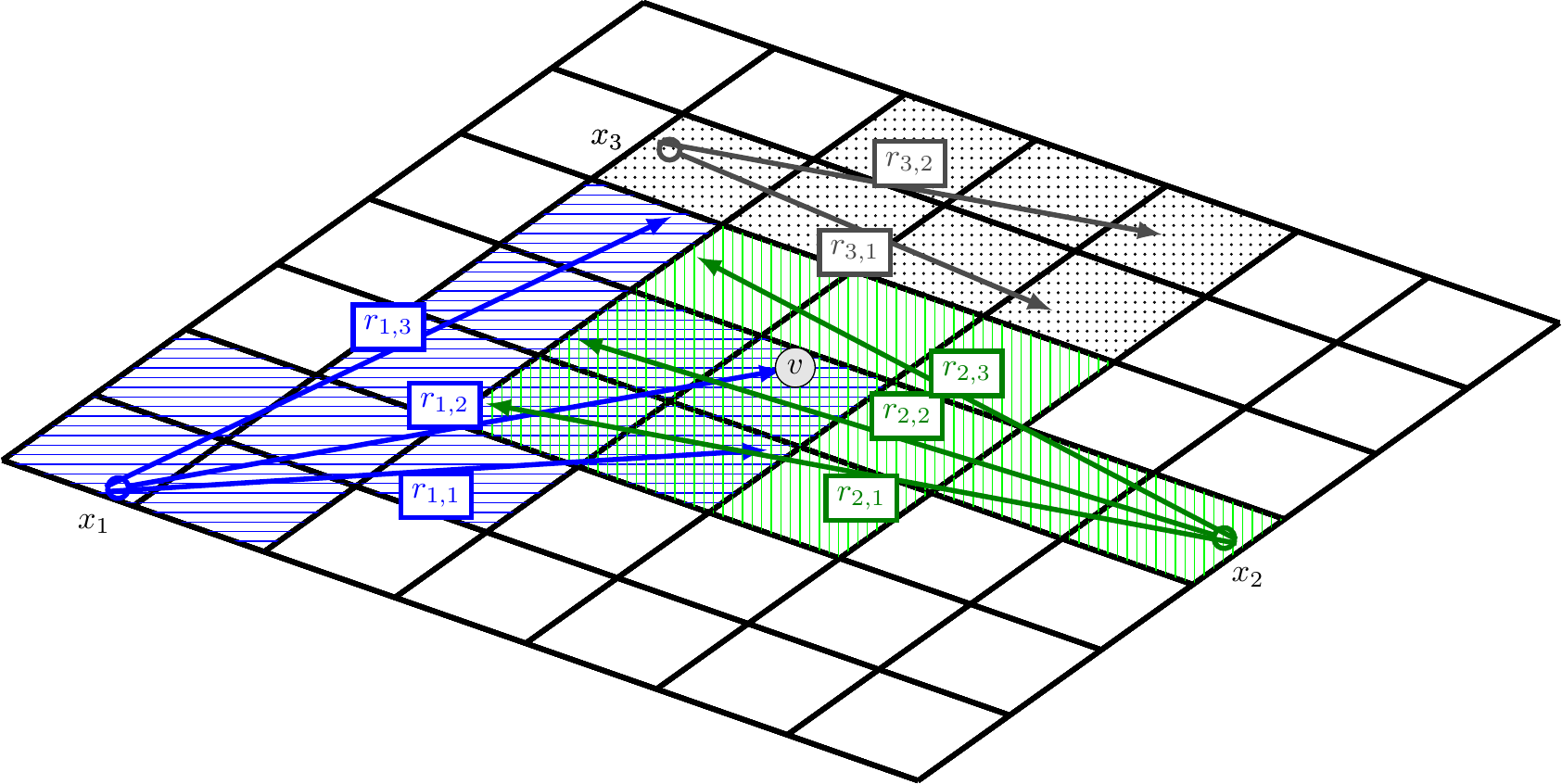}
  \caption{Grid cells are shown by the black squares.
  Views are indicated by $x_i$. The rays $r_{i,j}$ emitted from $x_i$ are indicated by colored arrows. The colored shading on a grid cell indicates that the grid cell is traversed through by a ray with the corresponding color.}
  \label{fig:rays}
\end{figure}

Each of the sensors is a camera that outputs a depth image.
The view $x_i$ of each sensor is its translation and orientation w.r.t.~a fixed coordinate system.
A depth image is considered as a set of rays emitted from $x_i$.
A ray is emitted in the direction of each pixel in the depth image, terminating after traversing a distance equal to the measured depth.
We write each ray $r$ as a sequence of grid cells it traverses through, e.g., $r=(v_1, v_2, \ldots, v_k)$, where $v_i \in V$ and $k$ is the total number of grid cells traversed.
At view $x_i$, the set of all rays emitted by the sensor is $R(x_i)$.

Fig.~\ref{fig:rays} illustrates a volumetric grid.
For clarity the drawing is in two dimensions, but in this paper we deal with three dimensional grids.
Three views and their corresponding sets of emitted rays are shown.
For example, the set of rays emitted from $x_1$ is $R(x_1) = \{r_{1,1}, r_{1,2}, r_{1,3}\}$.
Each ray $r_{i,j}$ is also a sequence of grid cells. 
For example, $r_{1,2}$ is the sequence of all grid cells traversed by the ray starting at the cell where $x_1$ is located and terminating at grid cell $v$ labeled by the shaded circle marker near the center of the grid.

Given a real depth image recorded by a sensor, the occupancy probabilities are updated using the inverse sensor model of~\cite{Hornung2013} as follows.
Raytracing is applied to determine which grid cells are updated.
The focal lengths $(f_x, f_y)$ and the principal point $(c_x,c_y)$ of the depth camera are required.
Let $d$ be the measured depth at pixel coordinates $(x,y)$.
The ray direction vector $\begin{bmatrix}\nicefrac{(x-c_x)}{f_x}& \nicefrac{(y-c_y)}{f_y}& 1\end{bmatrix}^T$ is transformed to the coordinate system of the volumetric grid, and a ray of length $d$ is emitted in the resulting direction.
The grid cells traversed by the ray are recorded.
The grid cell where a ray terminates is observed as a hit, and other grid cells traversed by the ray are observed as misses.
The log-odds representation for $P(v)$ is $L(v) = \log \frac{P(v)}{1-P(v)}$.
Given an observed grid cell $v$ and its observation $z$, the log-odds are updated to $L(v\mid z) = L(v) + l(z)$, where
\begin{equation}
  l(z) = \begin{cases} \log \frac{p_h}{1-p_h} & \text{if } z = \text{hit}\\
  \log \frac{p_m}{1-p_m} & \text{if } z = \text{miss}
  \end{cases}
\end{equation}
and $p_h$ and $p_m$ are the hit and miss probabilities, respectively.
The posterior $P(v\mid z)$ is obtained by inverting $L(v\mid z)$.


\subsection{Score function for an emitted ray} 
\label{sub:score_function_for_emitted_rays}
The ray score function calculates the score of a single ray as a sum over the grid cells traversed by the ray.
At each grid cell $v$, a weight term dependent on which other cells the ray traversed to reach $v$ is multiplied by the information gain available at $v$.
We formalize this by the following definition, and give concrete examples afterwards.
\begin{definition}[Ray score function]
  Let $r$ be a ray that traverses the sequence $(v_1, v_2, \ldots, v_k)$ of grid cells.
  The ray score function $s:V^k \to \mathbb{R}$ is defined as
  \begin{equation}
    s(r) = \sum\limits_{j=1}^k w_{j,r}(v_1, \ldots, v_{j-1}) c(v_j),
  \end{equation}
  where $w_{j,r}:V^{j-1}\to\mathbb{R}$ is a weight term, and $c:V\to\mathbb{R}$ is the information gain available at a grid cell.
  For $j=1$, we define $w_{1,r} \equiv 1$.
\end{definition}

In our experiments we use the widely applied entropy score~\cite{Kriegel2015,delmerico2018comparison} by setting the information gain equal to entropy of voxel occupancy, $c(v) = -P(v)\log_2 P(v) -(1-P(v))\log_2(1-P(v))$, and defining $w_{j,r}$ always equal to one.
This choice encourages exploration by assigning a high score for views that observe voxels with a high uncertainty.
However, we derive all our theoretical results for the general form given in the definition above, which includes many scores proposed in earlier literature.
Counting measures~\cite{Connolly1985,Banta2000} are obtained by setting the weight term always equal to 1, and the information gain function to return 1 when the grid cell is unknown and 0 otherwise.
Occlusion-aware scores~\cite{delmerico2018comparison} are obtained by setting the weight term equal to $\prod\limits_{i=1}^{j-1}(1-P(v_i))$, the probability that all grid cells traversed are free, i.e., the probability that the ray reaches $v_{j-1}$ before hitting an obstacle and terminating.
A region of interest $S \subset V$ is focused by setting the weight term to depend on the indicator function~of~$S$.

Single-sensor NBV approaches select a view $x$ by maximizing $\sum\limits_{r\in R(x)} s(r)$.
A naive extension of the single-sensor NBV approach to a multi-sensor setting selects for each sensor $i$ the view $x_i$ that maximizes $\sum\limits_{r\in R(x_i)} s(r)$.
However, since there is no incentive for coordination between the sensors, unnecessarily overlapping views may be selected.


\section{An overlap-aware utility function for multi-sensor next-best-view planning} 
\label{sec:overlap_aware_utility_function_for_sets_emitted_rays}
To coordinate view selection for multiple sensors, we propose an overlap-aware utility function to score sets of rays.
For each grid cell, the utility function only takes into account the ray along which the weighted information gain is maximized.
This captures the intuitive notion that overlap between multiple sensors should be avoided if possible, and only the most useful ray traversing through a particular grid cell should be considered in NBV planning.
Our proposed utility function helps reduce unnecessary overlap in the selected views, as we later show experimentally.
We prove our utility function is monotonically increasing and submodular, and show how single-sensor NBV planning arises as a special case when sensor views do not overlap.

\subsection{Overlap-aware utility} 
\label{sub:overlap_aware_utility}
Recall that $X_i$ are the pairwise disjoint sets of available views for each sensor.
Let $\Omega$ denote the set of all possible views obtained as the union of all $X_i$.
The partition matroid of valid views is $(\Omega, \mathcal{I})$ with $\mathcal{I} = \{I \subseteq \Omega \mid \forall i: |I \cap X_i| \leq 1 \}$.
Any independent set $I\in\mathcal{I}$ contains at most one view from each $X_i$.
The union of all rays emitted from views in an independent set is $R(I) = \cup_{x_i \in I}R(x_i)$.
For example in Fig.~\ref{fig:rays}, $R(\{x_1, x_2\}) = \{r_{1,1},r_{1,2},r_{1,3},r_{2,1},r_{2,2},r_{2,3}\}$.

The overlap-aware utility function considers for each grid cell all the rays that traverse through the cell.
For any grid cell $v$, we denote by $T_I(v)$ the subset of rays in $R(I)$ that traverse through $v$.
Recalling that any ray $r \in R(I)$ is a sequence of grid cells, we define
\begin{equation}
\label{eq:traverse_through_v}
  T_I(v) = \left\lbrace r \in R(I) \middle| r \cap \{v\} \neq \emptyset \right\rbrace.
\end{equation}
For instance, letting $R(\{x_1, x_2\})$ be as in the paragraph above, and considering the grid cell $v$ indicated in Fig.~\ref{fig:rays},  $T_{\{x_1, x_2\}}(v) = \{r_{1,2}, r_{2,2}\}$, containing exactly the two rays that traverse through $v$.

We propose the following overlap-aware utility function.
\begin{definition}[Overlap-aware utility function]
\label{def:overlap_aware}
  Let $(\Omega, \mathcal{I})$ be a partition matroid of valid sensor views.
  For any independent set $I \in \mathcal{I}$ of views, the overlap-aware utility is
  \begin{equation}
  \label{eq:definition_overlap_aware}
  f(I) = \sum\limits_{v\in V} \max\limits_{r \in T_I(v)} \left[ w_{j,r}(v_1, \ldots, v_{j_r - 1})c(v) \right],  
  \end{equation}
  where $w_{j,r}(v_1, \ldots, v_{j_r - 1})$ is the weight term, and $c$ is the information gain available at $v$.
  For $I=\emptyset$, set $f(\emptyset)=0$.
\end{definition}
At any grid cell, only the contribution from the ray that has the greatest weight term when reaching that grid cell is considered for overlap-aware utility.
For instance, if the weight term is chosen as the probability that the ray reaches the grid cell, only the ray with the greatest probability contributes to the utility.
Selecting overlapping views that observe the same grid cells is implicitly discouraged by the utility function.
Emitting an additional ray that reaches a grid cell only improves the utility if the additional ray has a greater weight term.
For example, consider the grid cell $v$ indicated by the gray circle in Fig.~\ref{fig:rays}.
The sum term in Eq.~\eqref{eq:definition_overlap_aware} that corresponds to $v$ is the maximum over the two rays $r_{1,2}$ and $r_{2,2}$ that traverse through~$v$.


\subsection{Proof of submodularity and monotonicity} 
\label{sub:overlap_aware_utility_is_submodular}
We prove a supporting lemma and then the main result.
\begin{lemma}
\label{lemma:submodularity}
For any $v \in V$, the set function
  \begin{equation}
    \label{eq:per_voxel_score}
    g_v(I) = \max\limits_{r\in T_I(v)} w_{j,r}(v_1, \ldots, v_{j_r-1})c(v)
  \end{equation}
  is monotonically increasing and submodular.
\end{lemma}
\begin{proof}
  Fix $v\in V$ and $A \subseteq B \subseteq \Omega$.
  Thus, $T_A(v) \subseteq T_B(v)$, since adding more views can only increase the number of rays that traverse through $v$.
  The maximum cannot decrease, i.e., $g_v(A) \leq g_v(B)$, so $g_v$ is monotonically increasing.

  To prove submodularity, fix $v\in V$ and $A \subseteq B \subseteq \Omega$.
  Select an arbitrary view $x \in \Omega \setminus B$, which is not in $B$ and thus also not in $A$.
  By the definition in Eq.~\eqref{eq:traverse_through_v}, every element in $T_{\{x\}}(v)$ is in both $T_{A\cup\{x\}}(v)$ and $T_{B\cup\{x\}}(v)$.
  Thus,
  \[
  T_{A\cup\{x\}}(v) \!=\! T_A(v) \cup T_{\{x\}}(v) \!\subseteq \!T_B(v) \cup T_{\{x\}}(v) \!=\! T_{B\cup\{x\}}(v),
  \]
  which implies $g_v(A\cup\{x\}) = \max \left\lbrace g_v(A), g_v(\{x\}) \right\rbrace$ and $g_v(B\cup\{x\}) = \max \left\lbrace g_v(B), g_v(\{x\}) \right\rbrace$.
  Now,
  \begin{align*}
    &g_v(A\cup\{x\}) - g_v(A) = \max \left\lbrace g_v(A), g_v(\{x\}) \right\rbrace - g_v(A)\\
    &=\max \!\left\lbrace 0, g_v(\{x\})-g_v(A) \right\rbrace \geq \max \!\left\lbrace 0, g_v(\{x\})-g_v(B) \right\rbrace\\
    &=\max \!\left\lbrace g_v(B), g_v(\{x\}) \right\rbrace - g_v(B)=g_v(B\cup\{x\}) - g_v(B),
  \end{align*}
  where the inequality follows as $g_v$ is monotonically increasing and the proof for submodularity is complete.
\end{proof}
\begin{theorem}
  The overlap-aware utility function $f$ as defined in Eq.~\eqref{eq:definition_overlap_aware} is submodular and monotonically increasing.
\end{theorem}
\begin{proof}
A sum of monotonically increasing submodular terms (Lemma~\ref{lemma:submodularity}) is monotonically increasing and submodular.
\end{proof}
Recall that a partition matroid with blocks $X_i$ describes the valid combinations of views the robots may choose.
We proved that the overlap-aware utility function $f$, Definition~\ref{def:overlap_aware}, is a submodular function of the independent sets of this partition matroid.
In summary, multi-sensor NBV planning, Eq.~\eqref{eq:problem} is matroid-constrained submodular maximization.



\subsection{Single-sensor NBV planning is a special case} 
\label{sub:analysis}
We prove that if there is no potential overlap between the views of any sensors, multi-sensor NBV planning reduces to solving $n$ independent single-sensor NBV planning problems.
Two views are \emph{disjoint} if the rays emitted from the views do not traverse through any of the same grid cells.
\begin{definition}
  Two views $x_i, x_j$ are disjoint if for every grid cell $v\in V$, the rays emitted from $x_i$ and $x_j$ do not overlap, that is, $\forall v: T_{\{x_i\}}(v) \cap T_{\{x_j\}}(v) = \emptyset$.
\end{definition}
In Fig.~\ref{fig:rays}, the views $x_1$, $x_3$ are disjoint, and the views $x_2$, $x_3$ are disjoint, but the views $x_1$, $x_2$ are not disjoint.

The following proposition shows that if there is no overlap between the views of any two different sensors, the multi-sensor NBV problem with the overlap-aware utility function is equivalent to $n$ single-sensor NBV planning problems.
\begin{proposition}
\label{thm:disjoint}
  Let $X_i$ be pairwise disjoint sets of views for each sensor $i$.
  If for every $i \neq j$, all $x_i \in X_i$ and $x_j \in X_j$ are disjoint, then Eq.~\eqref{eq:problem} with the overlap-aware utility function from Eq.~\eqref{eq:definition_overlap_aware} is equivalent to $\sum\limits_{i=1}^n \max\limits_{x_i \in X_i} f(\{x_i\}).$
\end{proposition}
\begin{proof}
Let $I = \{x_1, \ldots, x_n\}$.
As all views are disjoint,
for any $v \in V$, there exists exactly one $x_i \in I$ such that $T_I(v) = T_{\{x_i\}}(v)$.
The claim is proven by rearranging Eq.~\eqref{eq:definition_overlap_aware}: 
$\sum\limits_{i=1}^{n} \sum\limits_{v\in V} \max\limits_{r\in T_{\{x_i\}}(v)} \left[ w_{j,r}(v_1, \ldots, v_{j_r-1})c(v)\right] = \sum\limits_{i=1}^{n} f(\{x_i\}).$
\end{proof}


\section{A greedy algorithm for multi-sensor NBV planning} 
\label{sec:greedy_maximization_for_multi_sensor_nbv_planning}
Algorithm~\ref{alg:greedy} is a greedy strategy for solving Eq.~\eqref{eq:problem}.
The algorithm starts from the initial solution $I_0 = \emptyset$ at iteration $k=0$.
As long as there exists a view $x$ in the ground set $\Omega$ such that $I_k \cup \{x\}$ is an independent set, we repeat the following two steps.
First, find the view $x^*$ with the greatest marginal utility that maintains the matroid constraint, that is,
\begin{equation}
  x^* = \argmax\limits_{x\in \Omega \text{ s.t. } I_{k}\cup \{x\}\in \mathcal{I}} f(I_k \cup \{x\}) - f(I_k).
\end{equation}
Second, update the solution by $I_{k+1} = I_k \cup \{x^*\}$, remove $x^*$ from the ground set, and increment $k$.
When the input is the partition matroid of sensor views, the output $I_n$ contains exactly one view for each of the $n$ sensors.
When $f$ is monotonically increasing and submodular, $I_n$ is a $\frac{1}{2}$-approximation~\cite{fisher1978analysis}, that is, $f(I_n) \geq \frac{1}{2}\max\limits_{I\in\mathcal{I}} f(I)$.

\begin{algorithm}[t]
\caption{Greedy multi-sensor NBV planning}
\label{alg:greedy}
\begin{algorithmic}[1]
\Require{Partition matroid $(\Omega, \mathcal{I})$ of views}
\Ensure{Independent set $I_k\in \mathcal{I}$ containing $k$ planned views}
  \State $k \gets 0$, $I_k \gets \emptyset$
  \While{$\exists x \in \Omega: I_k \cup \{x\} \in \mathcal{I}$}
    \State $x^* \gets \argmax\limits_{x \in \Omega \text{ s.t. } I_{k}\cup \{x\}\in \mathcal{I}} f(I_k \cup \{x\}) - f(I_k)$ \label{line:max}
    \State $I_{k+1} \gets I_{k} \cup \{x^*\}$, $\Omega \gets \Omega \setminus \{x^*\}$, $k \gets k+1$
  \EndWhile
  \State \Return $I_k$
\end{algorithmic}
\end{algorithm}

\begin{figure*}[!ht]
  \centering
  \begin{tabular}{ccccccccc}
  \scriptsize
  \includegraphics[width=1.8cm]{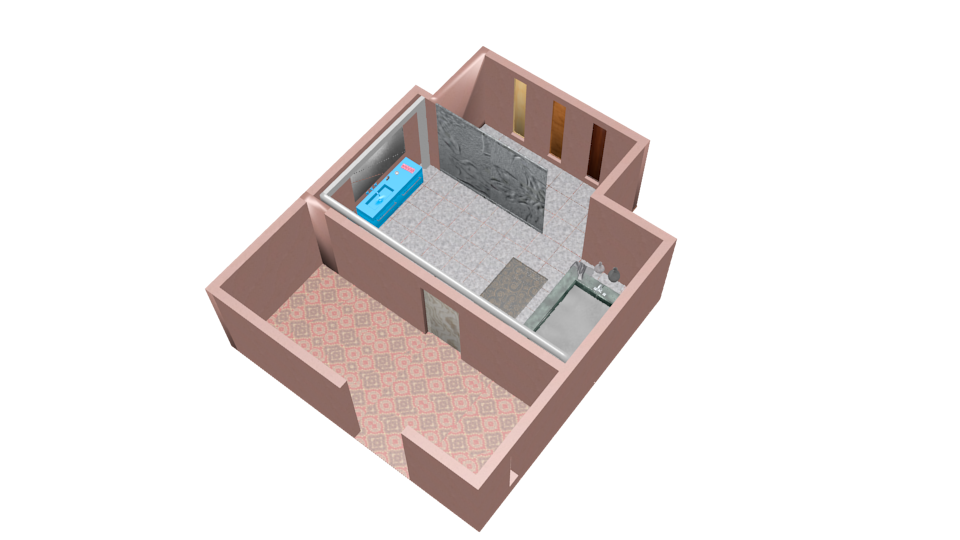} & \includegraphics[width=1.8cm]{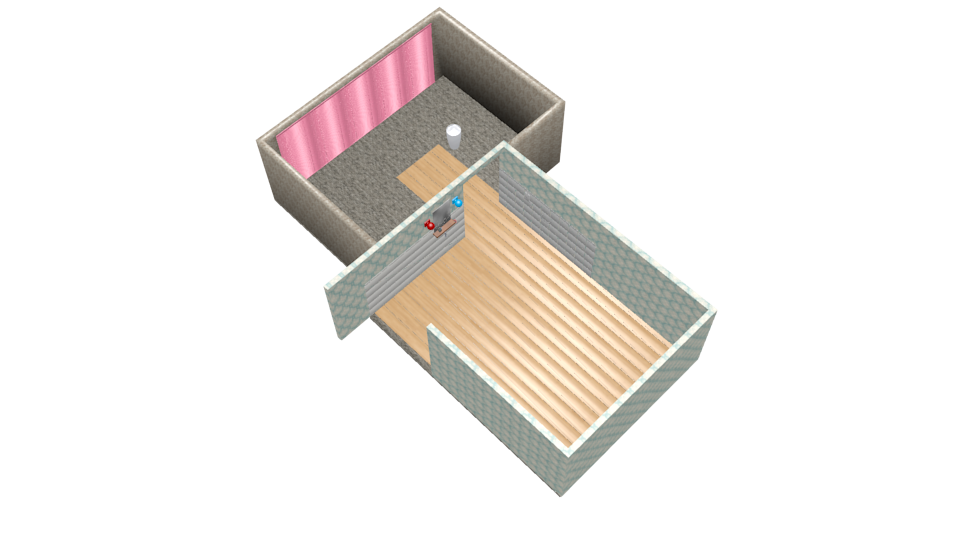} & \includegraphics[width=1.8cm]{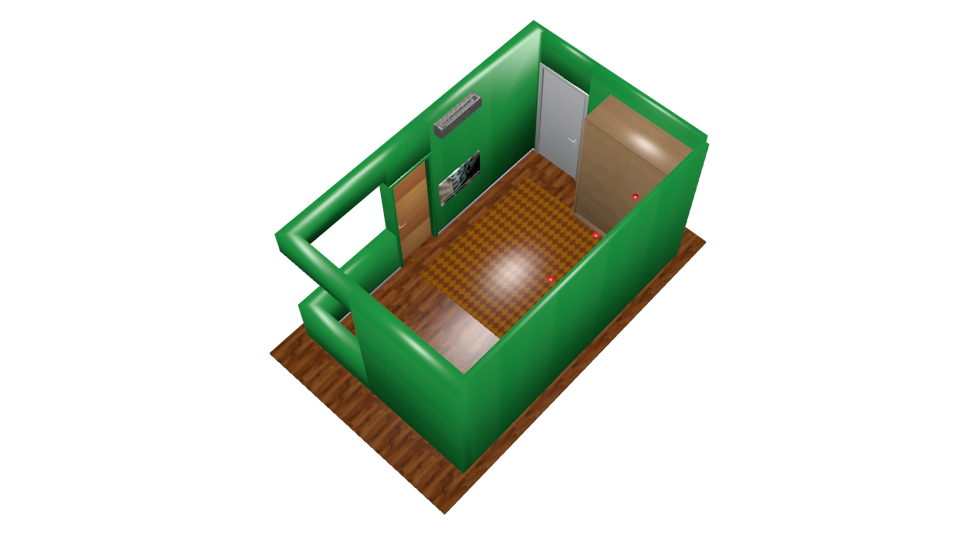} & \includegraphics[width=1.8cm]{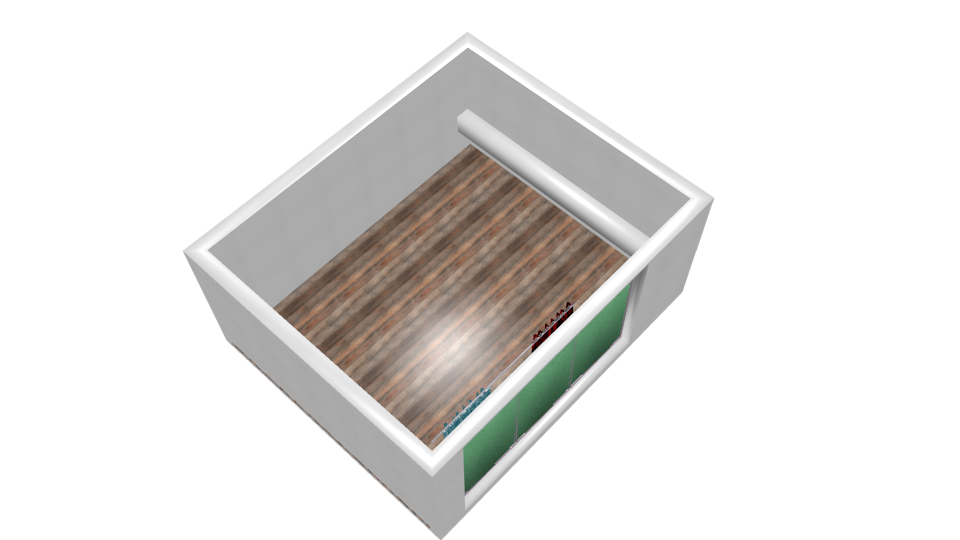} & \includegraphics[width=1.8cm]{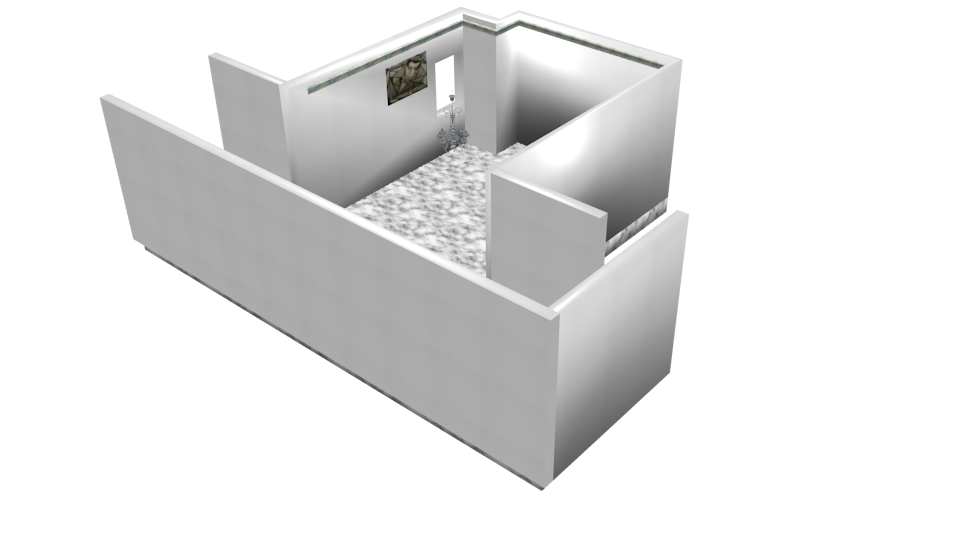} & \includegraphics[width=1.8cm]{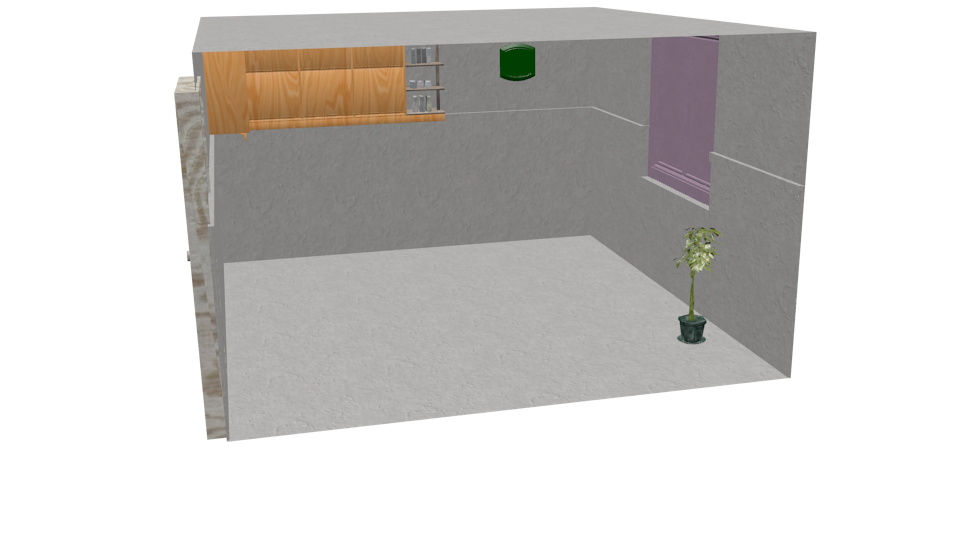} & \includegraphics[width=1.8cm]{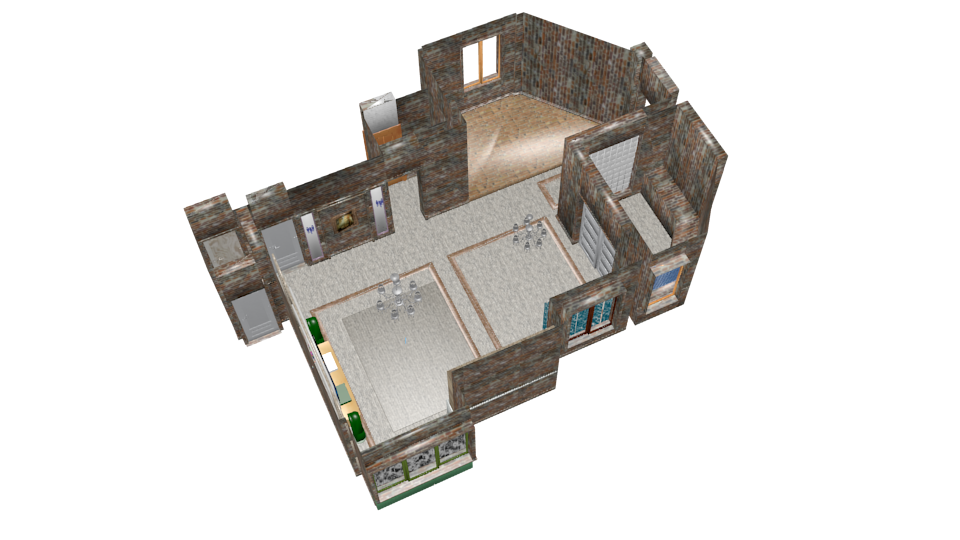} & \includegraphics[width=1.8cm]{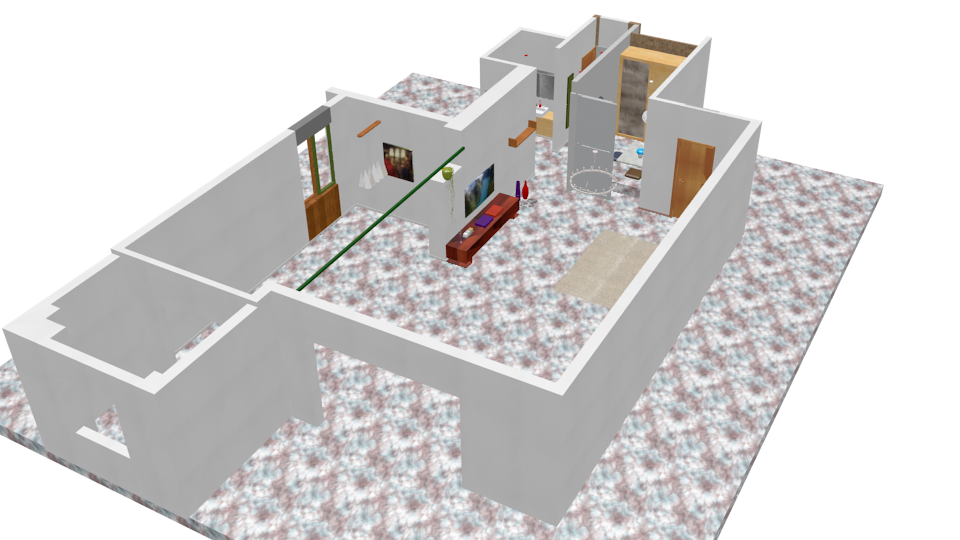} \\ 
  \scriptsize Bathroom 8 & \scriptsize Bathroom 9 & \scriptsize Bedroom 4 & \scriptsize Bedroom 11 & \scriptsize Kitchen 6  & \scriptsize Kitchen 11 & \scriptsize Living room 3 & \scriptsize Living room 5 \\ 
  \end{tabular}
  \caption{The types of layouts in the simulation experiments.}
  \label{fig:environments}
\end{figure*}
\begin{table*}[!ht]
\caption{Average area under curve for surface coverage in the simulation experiments. Best average values bolded.}
\label{tab:auc_surface_coverage_synthetic}
\begin{center}
\begin{tabular}{@{}llccccccccc@{}}
\toprule
                           & Method & Avg & Bathroom\! 8 & Bathroom\! 9 & Bedroom\! 4 & Bedroom\! 11 & Kitchen\! 6 & Kitchen\! 11 & Living\! room\! 3 & Living\! room\! 5 \\ \midrule
\multirow{3}{*}{2 cameras} & Ours   & \textbf{80.1}& 82.3       & 79.0       & 76.1      & 85.1       & 82.6      & 77.4       & 78.7          & 77.0          \\
                           & Single & 75.2         & 77.1       & 73.7       & 72.0      & 80.3       & 77.5      & 72.7       & 73.0          & 71.3          \\
                           & Random & 72.0         & 71.2       & 70.4       & 72.1      & 75.8       & 74.2      & 75.6       & 70.6          & 69.4          \\ \midrule
\multirow{3}{*}{4 cameras} & Ours   & \textbf{89.2}& 90.5       & 88.6       & 86.7      & 92.1       & 90.7      & 87.6       & 87.5          & 86.7          \\
                           & Single & 84.5         & 85.9       & 83.6       & 82.6      & 87.9       & 86.1      & 82.7       & 82.2          & 81.4          \\
                           & Random & 83.9         & 82.2       & 81.9       & 83.8      & 86.0       & 84.7      & 86.3       & 81.8          & 81.1          \\ \midrule
\multirow{3}{*}{8 cameras} & Ours   & \textbf{94.6}& 95.1       & 94.2       & 93.4      & 95.9       & 95.4      & 94.3       & 93.1          & 92.9          \\
                           & Single & 91.9         & 92.5       & 91.3       & 91.0      & 93.6       & 92.8      & 91.4       & 89.8          & 89.4          \\
                           & Random & 91.4         & 90.4       & 90.0       & 91.6      & 92.5       & 91.9      & 93.2       & 89.6          & 89.3          \\ \bottomrule
\end{tabular}
\end{center}
\end{table*}

\begin{figure}[!ht]
  \centering
  \includegraphics[width=0.5\columnwidth]{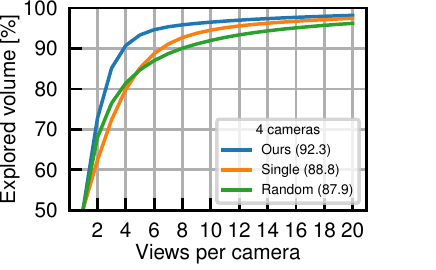}~\includegraphics[width=0.5\columnwidth]{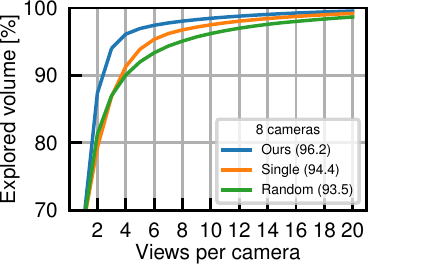}
  \caption{Volume explored as a function of the number of views per camera for 4 cameras (left) and 8 cameras (right) in the simulation experiments. Note the different vertical axis scales.}
  \label{fig:volume_48}
\end{figure}

\begin{figure}[!ht]
  \centering
  \includegraphics[width=0.5\columnwidth]{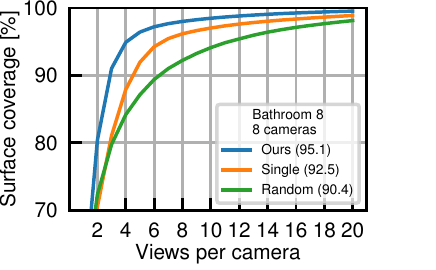}~\includegraphics[width=0.5\columnwidth]{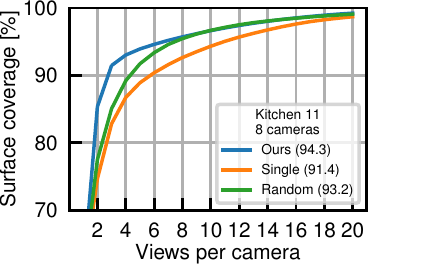}
  \caption{Surface coverage as a function of the number of views per each of 8 cameras in layouts Bathroom~8 (left) and Kitchen~11 (right) in the simulation experiments.}
  \label{fig:surface_comparison}
\end{figure}

We perform raytracing for each view $x\in \Omega$ once and store the best per-voxel scores $g_v(\{x\})$.
We initialize $g_v(I_0)$ to zero for all voxels.
The marginal utility of a view $x$ on Line~\ref{line:max} is then computed as follows.
Initialize the marginal utility to zero.
Loop over voxels visible for $x$ and compare the stored value $g_v(\{x\})$ to the old value $g_v(I_k)$.
If the stored value is greater than the old value, the positive difference is added to the marginal utility.
If the old value is greater than the stored value, continue to the next visible voxel.
After the view $x^*$ with the greatest marginal utility is recovered, the old values for visible voxels are updated to $g_v(I_{k+1}) = \max\{ g_v(I_k), g_v(\{x^*\})\}$.
If raytracing for one view has computational complexity $O(R)$, the overall computational complexity of Algorithm~\ref{alg:greedy} is $O(R|\Omega| + n|\Omega|)$ where the second term accounts for computing the marginal utilities of up to $|\Omega|$ views for up to $n$ sensors.

The proposed greedy submodular maximization algorithm provides an approximation guarantee and runs in polynomial time.
Eq.~\eqref{eq:problem} could also be solved by exhaustive search or a heuristic algorithm.
Exhaustive search recovers an optimal solution but has a worst-case computational complexity of $O(R|\Omega| + |\Omega|^n)$ as the utility of every independent set must be evaluated.
Heuristic algorithms that start with an initial solution and modify it locally are also applicable.
Examples of such algorithms include genetic optimization algorithms, tabu search, or simulated annealing.


\section{Simulation experiments} 
\label{sec:simulation_experiments}
We evaluate our NBV planning algorithm in synthetic environments and in a robotic setup. This section presents the results for the synthetic environments.

\subsection{Experimental setup} 
\label{sub:experimental_setup}
We use the validation set of SceneNet RGB-D~\cite{McCormac2017}, a collection of rendered RGB-D images from randomly generated synthetic environments.
The validation set consists of 1000 scene configurations.
Each scene configuration depicts a synthetic room layout with randomly sampled objects in physically plausible poses.
The room layout is one of 8 possible choices\footnote{We omit the ``Office 14'' layout as it has only one scene configuration.} illustrated in Fig.~\ref{fig:environments}.
Each scene configuration has 300 labeled cameras poses along a continuous trajectory.
320-by-240 pixel depth images from the labeled poses are available to compute a ground truth scene reconstruction.

We consider $n=2$, 4, or 8 sensors, and split the camera poses in each scene configuration into $n$ disjoint subsets $X_i$.
Each subset contains possible views for sensor $i$ and spans the entire trajectory.
We plan a sequence of $20$ views for each sensor by Algorithm~\ref{alg:greedy}.
We compare to independent single-sensor NBV planning of views as described in Subsection~\ref{sub:score_function_for_emitted_rays}, and to selecting the next views uniformly at random.
Each experiment is repeated 10 times.
The true depth images are applied to update occupancy probabilities as described in Subsection~\ref{sub:environment_and_sensor_models}.
We use hit and miss probabilities $p_h=0.9$ and $p_m=0.1$, respectively.
We use a voxel grid with a resolution of \SI{0.05}{\meter} per voxel.
The initial occupancy probability is $P(v)=0.5$ for all voxels.
To vary the prior information about the scene, the first views are sampled randomly in each experiment and are the same for all methods.
We use a resolution of 0.1 rays per pixel and a maximum range of \SI{10}{\meter} for raytracing.

We record the explored volume and the surface coverage as a function of the number of views per camera.
We compare these values to the ground truth model created from all images.
For computing the surface coverage, a point in the ground truth model is considered observed if there is a point closer than \SI{0.05}{\meter} to it in the reconstruction.


\subsection{Results} 
\label{sub:results}

Fig.~\ref{fig:volume_48} shows for 4 and 8 cameras the average fraction of explored volume as a function of number of views.
The numbers in parentheses in the legend show the micro-averaged area under curve (AUC) value of the respective methods.
Our algorithm reaches 90\% explored volume with between 2-3 fewer views per camera compared to the other two methods.
Single-sensor planning does not coordinate view selection of multiple sensors, and prefers strongly overlapping views in the first 2-4 steps.
Random view selection has no preference for such views, and thus performs better during these steps.
After strongly overlapping views are removed from future consideration, single-sensor planning avoids selecting useless views on subsequent steps and outperforms random over the entire span of 20 views.
Our algorithm avoids overlapping views, outperforming single-sensor planning and random view selection for any number of selected views.

Table~\ref{tab:auc_surface_coverage_synthetic} shows the overall micro-averaged AUC for surface coverage, and the AUC for each layout.
On average, our algorithm observes more surface points than single-sensor planning or randomly selecting views.
As the number of cameras increases, the amount of potential overlap between the views increases.
With 4 or 8 cameras, the performance of single-sensor planning and random are similar, while ours performs significantly better.
This shows our algorithm is able to successfully coordinate views with increasing amount of potential overlap.
Our method performs best in layouts with many rooms and occluding walls, such Bathroom~8, Kitchen~6, and Living room~5 (see Fig.~\ref{fig:environments}).
The improvement is smallest in layouts consisting of a single room such as Kitchen~11, especially when the number of cameras is 8.
Single planning with 4 or 8 cameras also performs almost identically to random in both of the living room layouts.

Fig.~\ref{fig:surface_comparison} shows the surface coverage with 8 cameras in Bathroom~8 and Kitchen~11.
Numbers in parentheses in the legend show the AUC.
Our method performs best for any number of views in the Bathroom~8 layout that has many connected rooms separated by walls that occlude views (see Fig.~\ref{fig:environments}).
In the Kitchen~11 layout consisting of a single room with no occluding walls, our method still performs best with less than 8 views per camera, and then equally well as random view selection.
Single-sensor planning selects strongly overlapping views and performs worse than randomly selecting views.

For both planning algorithms, raytracing to compute the per-voxel scores takes most of the runtime.
The average runtime was \SI{0.8}{\second} per candidate view.


\section{Real-world experiments} 
\label{sec:real_world_experiments}
We set up three scenes: a scene with light clutter (Fig.~\ref{fig:real_scenes}, left), the same scene with a large obstacle added (Fig.~\ref{fig:real_scenes}, right), and a scene with unsorted metal waste (Fig.~\ref{fig:online_scene}, top).
Intel RealSense D435 depth cameras are attached to the two KUKA LBR iiwa robot arms.
For each robot, we sample 20 candidate views around the workspace.
In each view, the camera points to the center of the workspace.
We plan a sequence of 8 views by the proposed multi-sensor NBV planning (Ours), single-sensor NBV planning (Single), or random view selection (Random).
As there are only two cameras, for Ours we find the independent set that maximizes overlap-aware utility directly by evaluating all combinations.
Other experimental settings are as in Section~\ref{sec:simulation_experiments}.

Table~\ref{tab:real_unknown_volume} shows the average unknown volume\footnote{Unknown volume never reaches zero as object insides are unobservable.} (lower is better) and its 95\% confidence interval, after 3, 5, and 8 views in each scene.
In the ``Waste'' scene, Ours outperforms Single.
The views in the scene overlap strongly, and our method avoids redundant views.
In the ``Clutter'' and ``Obstacle'' scenes, there is no significant difference between Ours and Single.
In the ``Obstacle'' scene, the difference of Single, Ours, and Random after 8 views is not significant.
Randomly selecting 8 views is sufficient for a good reconstruction.
In all cases, the random baseline performs worst.

\begin{table}[t]
\caption{Average unknown volume (\SI{}{\cubiccentimetre}) and its 95\% confidence interval in the real-world scenes. Bolded values show all statistically significantly best-performing methods.}
\label{tab:real_unknown_volume}
\begin{center}
\begin{tabular}{@{}llccc@{}}
\toprule
Scene                    & Method  & After 3 views & After 5 views & After 8 views  \\ \midrule
\multirow{3}{*}{Clutter} & Ours & \textbf{11404 $\pm$ 128} & \textbf{7016 $\pm$ 55} & \textbf{5064 $\pm$ 55} \\
                         & Single      & \textbf{11592 $\pm$ 89} & \textbf{7104 $\pm$ 222} & \textbf{4888 $\pm$ 55} \\
                         & Random  & 18844 $\pm$ 3267 & 10018 $\pm$ 1009 & 5869 $\pm$ 426 \\ \midrule
\multirow{3}{*}{Obstacle}& Ours & \textbf{17360 $\pm$ 400} & \textbf{13592 $\pm$ 55} & \textbf{11628 $\pm$ 83} \\
                         & Single      & \textbf{17696 $\pm$ 277} & \textbf{13556 $\pm$ 6} & \textbf{11492 $\pm$ 128} \\
                         & Random  & 24627 $\pm$ 3960 & 15954 $\pm$ 1226 & \textbf{12064 $\pm$ 390} \\ \midrule
\multirow{3}{*}{Waste}   & Ours & \textbf{14296 $\pm$ 133} & \textbf{6400 $\pm$ 0} & \textbf{4476 $\pm$ 39}  \\
                         & Single      & 16584 $\pm$ 33 & 9260 $\pm$ 139 & 4984 $\pm$ 11 \\
                         & Random  & 20274 $\pm$ 2887 & 10549 $\pm$ 1144 & 5210 $\pm$ 547  \\ \bottomrule
\end{tabular}
\end{center}
\end{table}

\begin{figure}[t]
  \centering
  \includegraphics[width=0.5\columnwidth]{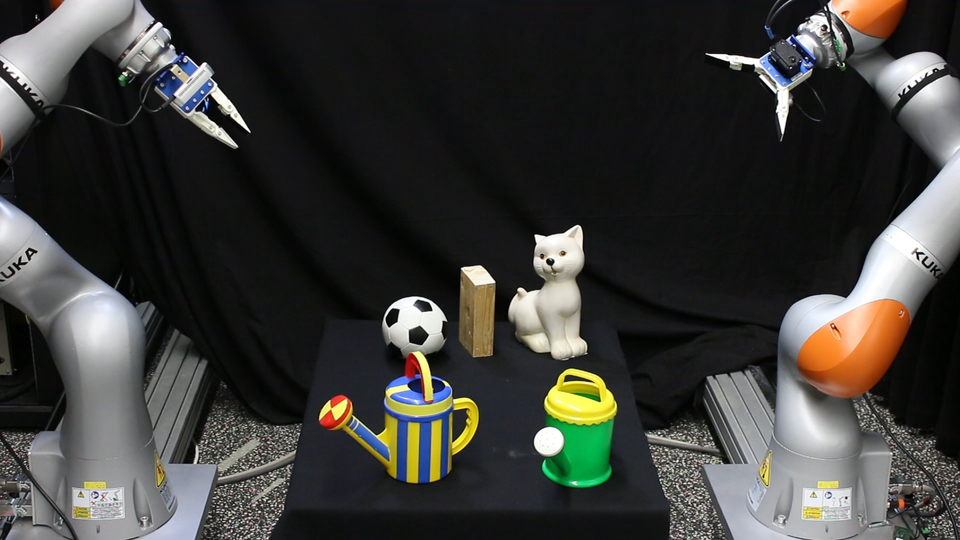}~
  \includegraphics[width=0.5\columnwidth]{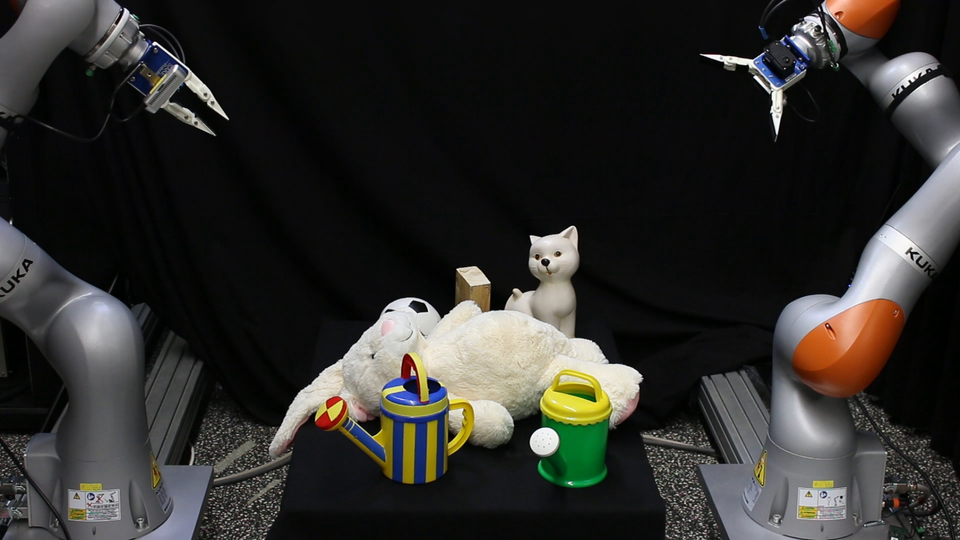}
  \caption{The clutter scene (left) and the obstacle scene (right).}
  \label{fig:real_scenes}
\end{figure}

\begin{figure*}[!ht]
  \centering
  \subfloat[Initial view]{\includegraphics[width=0.33\textwidth]{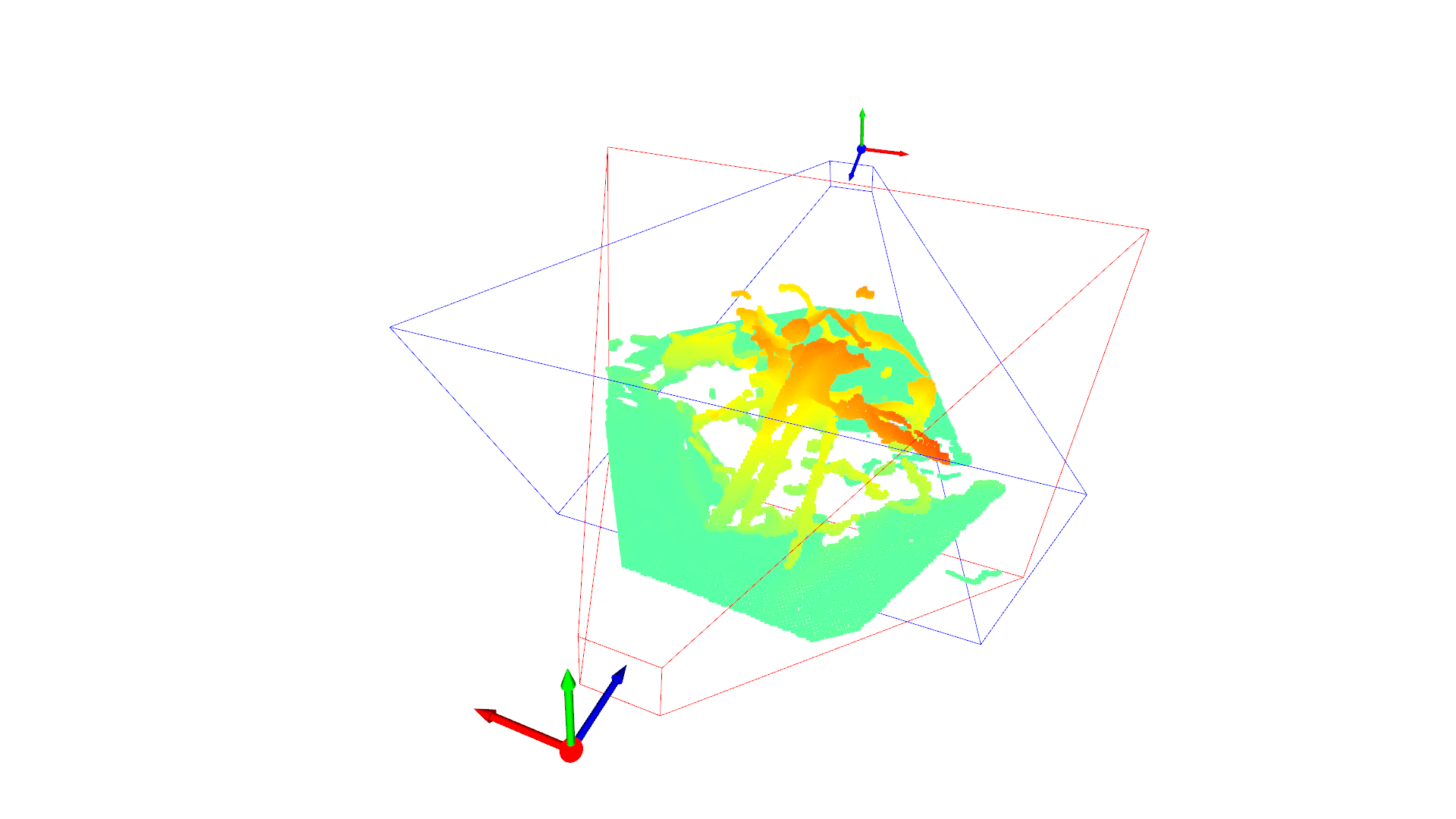}}%
  \subfloat[Ours]{\label{fig:waste_ours}\includegraphics[width=0.33\textwidth]{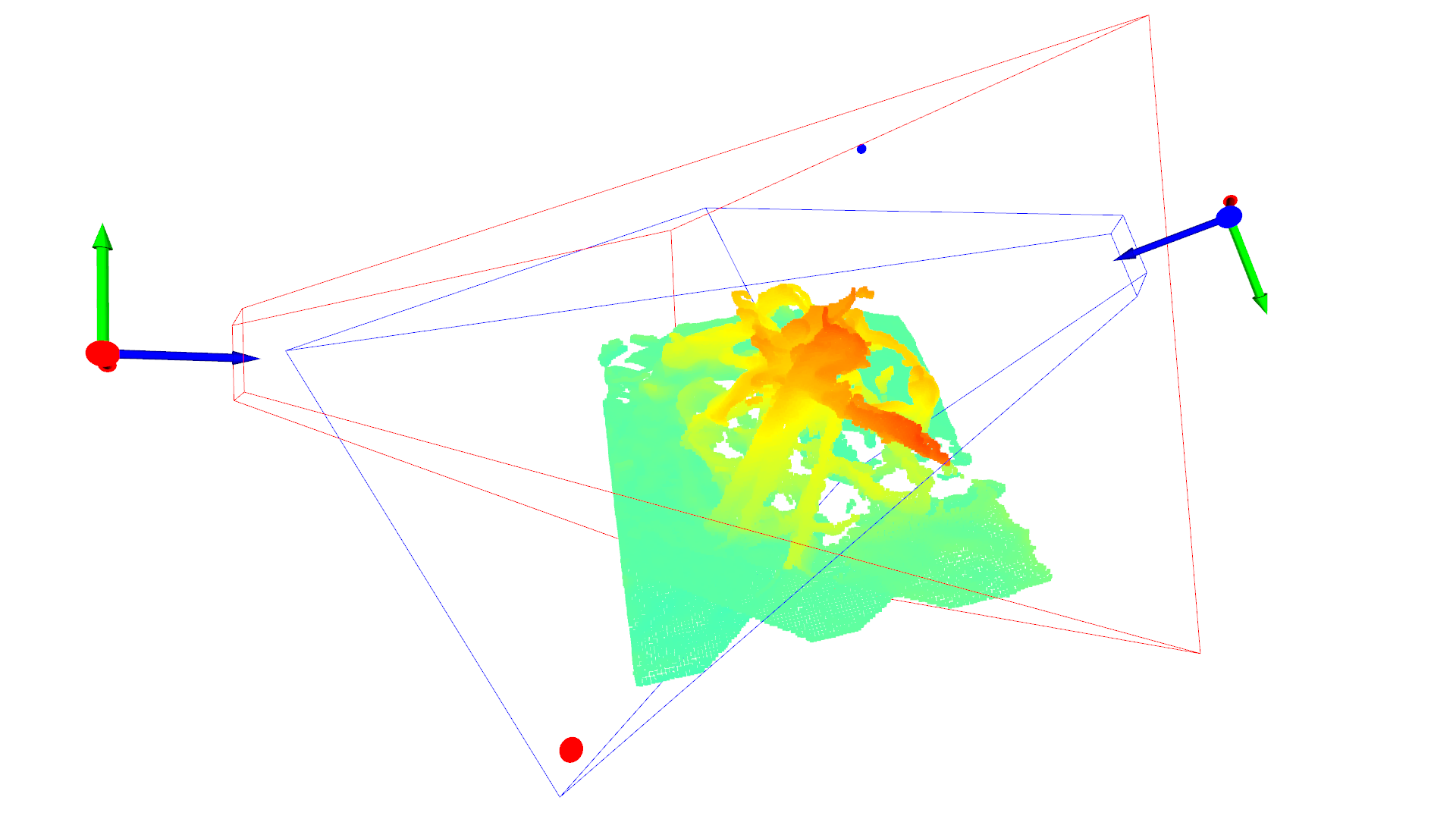}}%
  \subfloat[Single]{\label{fig:waste_ve}\includegraphics[width=0.33\textwidth]{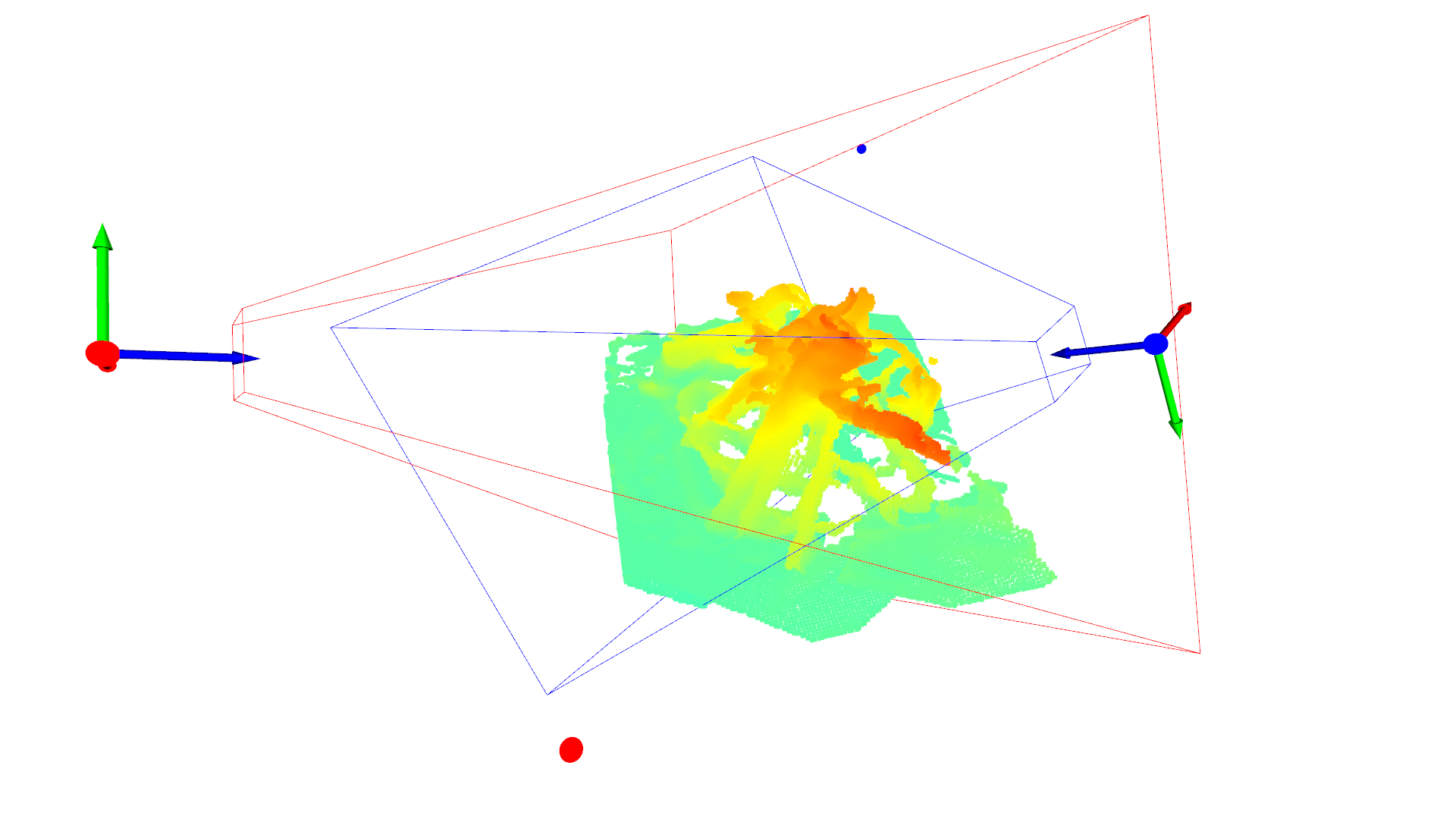}}
  \caption{Qualitative results in the ``Waste'' scene. The left subfigure shows the initial reconstruction and views (red and blue frustum). The middle and right subfigures show the second views selected by Ours and Single, respectively. The reconstruction color corresponds to height above the tabletop. The coordinate axes indicate the current pose of each sensor. Compared to Single, our method selects a view for sensor 2 (blue frustum) that avoids overlap with sensor 1 (red frustum) and observes the region to the bottom left.}
  \label{fig:waste_comparison}
\end{figure*}

In all scenes, Ours selected the same sequence of views in each repetition.
This was the case for Single as well.
In Table~\ref{tab:real_unknown_volume} the variation for Ours and Single is due to unpredictable events such as localization or sensor noise.

Fig.~\ref{fig:waste_comparison} qualitatively compares Ours and Single.
Our method explicitly considers that the two sensors' fields of view overlap, and selects views with less overlap.
Single maximizes utility for each sensor independently, and selects views that overlap resulting in less explored volume.
After the two views, the remaining unknown volume was \SI{34000}{\cubiccentimetre} (Ours) and \SI{38000}{\cubiccentimetre} (Single).

\section{Discussion and Conclusion} 
\label{sec:conclusion}
Several topics beyond the scope of this paper remain unexplored.
Future work could investigate how the granularity of the candidate view sampling affects performance.
Considering more candidate viewpoints is guaranteed to improve performance, but it might not be beneficial in practice due to the increased computational cost.
Dynamic generation of candidate viewpoints could also be considered, e.g., similar to~\cite{Bircher2016}.
Many formulations of information gain have been proposed and compared in single-sensor object reconstruction~\cite{delmerico2018comparison}, while we applied only the entropy score.
Comparing information gain formulations in the multi-sensor case is a direction for future work.
The demands for NBV planning vary depending on the application~\cite{Border2018,Sukkar2019,Cui2019}, and studying the quality of the reconstructions with respect to the parameters of the algorithm can help fit our proposed method to specific applications.

In conclusion, we propose a monotonically increasing and submodular overlap-aware utility function for multi-sensor next-best-view planning with a volumetric environment representation.
Our greedy planning algorithm is guaranteed to produce a solution within a constant factor from the optimum.
Experimental results show the algorithm avoids planning overlapping views, outperforms independent single-sensor planning and random view selection methods, and can be implemented in a real-world robotic system.

\ifCLASSOPTIONcaptionsoff
  \newpage
\fi



%



\bibliographystyle{IEEEtran}
\bibliography{refs}

%








\end{document}